\documentclass{article}

    \PassOptionsToPackage{numbers, compress}{natbib}
 \usepackage[main, final]{neurips_2026}
\usepackage[utf8]{inputenc} % allow utf-8 input
\usepackage[T1]{fontenc}    % use 8-bit T1 fonts
\usepackage{hyperref}       % hyperlinks
\usepackage{url}            % simple URL typesetting
\usepackage{booktabs}       % professional-quality tables
\usepackage{amsfonts}       % blackboard math symbols
\usepackage{nicefrac}       % compact symbols for 1/2, etc.
\usepackage{microtype}      % microtypography
\usepackage{xcolor}         % colors
\usepackage{graphicx}
\usepackage{amsthm}
\usepackage{amssymb}
\newtheorem{theorem}{Theorem}[section]
\newtheorem{proposition}[theorem]{Proposition}

\theoremstyle{plain}

\theoremstyle{definition}
\newtheorem{definition}[theorem]{Definition}

\newtheorem{remark}[theorem]{Remark}

\usepackage{amsmath}
\DeclareMathOperator{\E}{\mathbb{E}}
\newcommand{\cX}{\mathcal{X}}
\newcommand{\cY}{\mathcal{Y}}

\newcommand{\cA}{\mathcal{A}}

\newcommand{\Lpreq}{L_{\mathrm{preq}}}
\newcommand{\Lfinal}{\hat{L}}

\newcommand{\indicator}{\mathbf{1}}

\newcommand{\KL}{\mathrm{KL}}
\newif\ifshowchanges \showchangesfalse
\newcommand{\new}[1]{\ifshowchanges{\ifmmode\else\leavevmode\fi\color{red}#1}\else#1\fi}
\newcommand{\sym}[1]{\ifshowchanges{\color{red}#1}\else#1\fi}
\newcommand{\add}[1]{\ifshowchanges{\ifmmode\else\leavevmode\fi\color{red}#1}\else#1\fi}
\newcommand{\fix}[1]{\ifshowchanges{\ifmmode\else\leavevmode\fi\color{red}#1}\else#1\fi}
\title{Does Compression Imply Generalization? \\A Minimum Description Length Perspective}

\author{%
  Lincen Yang\textsuperscript{1} \quad
  Jiayang Shi\textsuperscript{2,1} \quad
  Nan Pu\textsuperscript{1,3} \quad
  Dani\"el M.~Pelt\textsuperscript{1} \\
  \textbf{Matthijs van Leeuwen}\textsuperscript{1} \quad
  \textbf{Zhong Li}\textsuperscript{4,1,}\thanks{Corresponding authors: Zhong Li
    (\texttt{z.li@liacs.leidenuniv.nl}) and Shujian Yu (\texttt{yusj9011@gmail.com}).} \quad
  \textbf{Shujian Yu}\textsuperscript{5,6,}\footnotemark[1] \\[1ex]
  \textsuperscript{1}Leiden Institute of Advanced Computer Science, Leiden University, The Netherlands \\
  \textsuperscript{2}Centrum Wiskunde \& Informatica (CWI), Amsterdam, The Netherlands \\
  \textsuperscript{3}Hefei University of Technology, Hefei, China \\
  \textsuperscript{4}Great Bay University, Dongguan, China \\
  \textsuperscript{5}Department of Computer Science, Vrije Universiteit Amsterdam, The Netherlands \\
  \textsuperscript{6}Department of Physics and Technology, UiT The Arctic University of Norway \\[1ex]
  \texttt{\{l.yang, j.shi, n.pu, d.m.pelt, m.van.leeuwen, z.li\}@liacs.leidenuniv.nl} \\
  \texttt{yusj9011@gmail.com}
}

\begin{document}

\maketitle

\begin{abstract}
A central question in modern learning theory is what notion of compression, if any, explains why over-parameterized neural networks generalize. Existing answers measure complexity at different levels, ranging from weights to representations; however, estimating the associated quantities either requires held-out data in principle or relies on the challenging (ultra) high-dimensional mutual information estimation. To overcome these challenges, we approach the question through the minimum description length (MDL) principle, and specifically, we propose the prequential regret as a complementary signal that is computable from the training trajectory alone. We establish that, in the (nearly) interpolating regime, the expected prequential regret upper-bounds the description length of the trained parameter under the algorithm-induced distribution, and that this description length translates into a high-probability generalization bound in the PAC-MDL setting. Along the way, we show that the encoder-complexity correction recently added to the algorithm-level information bottleneck (IB) is upper-bounded by an MDL parameter description length, \new{bringing together} the two views of compression from MDL and IB. Empirically, the prequential regret tracks the generalization error on CIFAR-10 and CIFAR-100, with Pearson correlations exceeding $0.9$.
\end{abstract}

\section{Introduction}
% ======================================================================
A fundamental question in statistical learning theory is why over-parameterized models generalize to unseen data, despite having enough capacity to memorize their training sets~\citep{zhang2016understanding, belkin2019reconciling}. Classical theory measures complexity at the level of the model class, through quantities such as VC dimension~\citep{vapnik2013nature}, Rademacher complexity~\citep{bartlett2002rademacher}, or norm-based capacity~\citep{bartlett2017spectrally}. By contrast, a modern line of work shifts attention to what the learned model actually retains from the training data~\citep{bousquet2002stability, russo2016controlling, xu2017information}. The complexity of the learned model is then characterized through information-theoretic quantities, and the central question becomes how to formalize and measure this complexity.

Three mainstream approaches make this complexity precise as a \emph{coding cost}. First, PAC-Bayes~\citep{lotfi2022pac, dziugaite2021role} bounds generalization by $D_{\mathrm{KL}}(Q\|\pi)$ between a posterior $Q$ over learned models and a prior $\pi$, which by the bits-back argument~\citep{hinton1993keeping, honkela2004variational} is the code length of transmitting a sample from $Q$ induced by $\pi$~\citep{grunwald2007minimum}. 
Second, the information-theoretic line~\citep{russo2016controlling, xu2017information} replaces the prior with the data-marginal and bounds generalization by $I(S;W)$, measuring how much information the weights $W$ retain about the training set $S$. Since $I(S;W)$ can be infinite for deterministic algorithms, the supersample/CMI variants~\citep{steinke2020reasoning, haghifam2020sharpened, harutyunyan2021information} introduce a pooled sample $\tilde Z$ of $2n$ points together with a uniform selector $U$ that assigns half to training and half to a `ghost' test set, and bound generalization by $I(W; U \mid \tilde Z)$, capturing how much the weights leak about the train/test assignment given the pool.
Third, the information bottleneck (IB)~\citep{tishby2000information, tishby2015deep, shwartz2017opening} asks how much the internal representations $Z$ retain about $X$, with $I(X;Z)$ or $I(X;Z\mid Y)$~\citep{tishby2000information} as the coding cost of the representation. Recent work observes that vanilla IB implicitly assumes a fixed encoder and can fail when the encoder\add{, i.e., the map from an input to its internal representation,} is itself learned from data~\citep{hafez2020sample}; \citet{kawaguchi2023does} hence addresses this by extending IB to the algorithm level, adding an encoder-complexity term that penalizes how much the trained encoder depends on the training set.

While these three lines capture different aspects of compression, they have the following shortcomings when we want a compression signal directly tied to fixed algorithms for training networks. On the conceptual side, what is being compressed is different across the literature: PAC-Bayes and the $I(S;W)$ line compress at the level of weights, while IB compresses at the level of representations, and the two notions have evolved separately, leaving it unclear whether they are facets of one underlying property or distinct phenomena that share a name. Further, even granting a notion of compression, whether Occam-style reasoning survives for over-parameterized neural networks is itself debated~\citep{zhang2016understanding, shah2020pitfalls, mingard2025deep}. 

On the practical side, the existing complexity measures are hard to estimate or evaluate in deep learning regimes. Specifically, IB and $I(S;W)$ require estimating mutual information for ultra-high-dimensional \new{feature vectors} $X$ or \new{weight parameters} $W$, the supersample/CMI variant requires a held-out twin sample, and PAC-Bayes requires a prior that captures the complex dependency among parameters. Thus, a complementary signal computable from the training set, with no extra data and no high-dimensional MI estimation, would shed light on generalization from the inside of training.

This motivates us to consider quantifying the model complexity from the alternative perspective of the minimum description length (MDL)~\citep{rissanen1978modeling, grunwald2007minimum, grunwald2019minimum}, which provides the tools to link the complexity of \new{a} model via the description length (in bits) needed to encode the model itself in a lossless manner. Specifically, to resolve the aforementioned shortcomings, \emph{we propose to use the prequential coding regret (defined in Section~\ref{section:mdl}) as an operational generalization bound, together with both theoretical justifications and empirical validations}. 

\textbf{Contributions.} We summarize our main contributions as follows.
\begin{itemize}
    \item We connect MDL parameter compression to the algorithm-level information bottleneck, showing that the encoder-complexity correction added to obtain the first rigorous IB-based generalization bound is upper-bounded by an MDL parameter description length, which relates the two views of compression from IB and MDL (Section~\ref{sec:ib}).
    \item We prove that, in the (nearly) interpolating regime, the expected prequential coding regret upper-bounds the description length of the trained parameter under the algorithm-induced distribution over parameters given the data (Section~\ref{section:regret-as-model-complexity})\new{.}
    \item We further show that this description length translates into a generalization bound in the PAC-MDL setting, with the algorithm-dependent content carried entirely by the prequential regret. We also empirically demonstrate \new{that} the regret tracks the generalization gap on CIFAR-10 and CIFAR-100 with Pearson correlations of $0.9$ or higher (\new{Sections}~\ref{sec:pacmdl} and~\ref{section:exp}).
\end{itemize}

\textbf{Related work.} Although MDL was incorporated into neural networks as early as the 1990s~\citep{hinton1993keeping}, the MDL-based model complexity of neural networks and its connection to generalization remain understudied, with a few exceptions as follows. First, the PAC-MDL bound~\citep{blum2003pac} connects the description length of a model to its generalization, yet, as in the PAC-Bayesian framework, leaves open how to quantify this description length. Second, \citet{blier2018description} empirically demonstrated that the prequential code can be used to effectively compress the data and the neural network together, and argued that the description length of the trained neural network was `contained' in the total code length; however, quantifying the description length of the model itself remains unclear. Third, recent theoretical work connects MDL with generalization bounds~\citep{sefidgaran2023minimum, sefidgaran2025generalization}, but the ``MDL'' there is defined as an expected KL divergence between the posterior of the latent representation and a symmetric conditional prior; while conceptually related, this does not directly measure the description length of the model. Fourth, the excess description length~\citep{donoway2026excess}, defined as the prequential code length minus the negative log-likelihood on a \emph{test} set, is conceptually close to the prequential regret but, by depending on the test set, does not measure model complexity from the training side alone.

{\textbf{Organization.} The remainder of the paper is organized as follows. Section~\ref{section:mdl} introduces the MDL principle and the prequential code, from which the prequential regret is defined. Section~\ref{sec:ib} reviews the recent correction to the information bottleneck and connects it to MDL, linking the two notions of compression. Section~\ref{section:regret-as-model-complexity} makes precise what the complexity of a trained network means and proves that it is bounded by the prequential regret, a quantity that can be computed from the training run alone. Section~\ref{sec:pacmdl} shows that this complexity leads to generalization guarantees, and Section~\ref{section:exp} tests empirically whether the regret tracks the generalization gap. }

\section{Preliminary: The MDL Principle} \label{section:mdl}
The minimum description length (MDL) principle, rooted in information theory, formulates the task of learning a model from data as a compression problem. That is, given a model class, the \new{``quality''} of a model is evaluated as how much it can compress the data, measured by the cost of encoding the model itself, plus the cost of encoding the data given the model~\citep{rissanen1978modeling, grunwald2007minimum}; the best model is hence the one that achieves the highest compression in this sense. It can be viewed as a practical approximation of Kolmogorov complexity~\citep{li2008introduction} and a formalization of Occam's razor~\citep{blumer1987occam}. 

We next briefly review the elements needed for our analysis. By Kraft's inequality~\citep{grunwald2007minimum}, prefix-free codes correspond to sub-probability distributions: any code with lengths $\ell(x)$ satisfies $\sum_x 2^{-\ell(x)} \le 1$, and conversely any distribution $Q$ defines a code of length $-\log Q(x)$, optimal up to one bit~\citep{cover:12:elements}. A single distribution thus determines a code. 

In the MDL literature, however, a \emph{model} refers not to a single distribution but to a family of probability distributions indexed by a parameter $\theta \in \Theta$, i.e., $M = \{Q_\theta\}_{\theta \in \Theta}$~\citep{grunwald2007minimum}. Encoding data under such a model is more delicate than the single-distribution case: the Kraft inequality applies to individual distributions, not to families, so the family alone does not determine a code. A \emph{universal code}~\citep{grunwald2019minimum} resolves this by inducing a single distribution over outcomes from the entire family, in such a way that its code length is nearly as short as $-\log Q_{\hat{\theta}}(y^{\sym{m}} \mid x^{\sym{m}})$, the code length of the best-fitting parameter $\hat{\theta}$ that corresponds to the shortest code length \new{in hindsight} (i.e., after seeing the data).

Two universal codes are central to our analysis: the \emph{two-part code} and the \emph{prequential code}, which we describe next. In both cases, we frame the encoding scheme as a protocol in which Alice transmits the labels $y^{\sym{m}}$ to Bob, who has access to the inputs $x^{\sym{m}}$.

\subsection{The two-part encoding scheme}
The \emph{two-part code}~\citep{grunwald2007minimum} is the most direct realization of the MDL principle: it splits the description of the data into one part for the model and one for the data given the model. Alice first transmits a parameter $\theta \in \Theta$ using some code with length $L(\theta)$, then transmits the labels using $Q_\theta$ at cost $-\log Q_\theta(y^{\sym{m}} \mid x^{\sym{m}})$ bits, choosing $\theta$ to minimize the total.
\begin{definition}[Two-part code length]
\label{def:twopart}
Given a code $L(\theta)$ over the parameter space $\Theta$, the \emph{two-part code length} is
\begin{equation}
  L_{\mathrm{2p}}(y^{\sym{m}} \mid x^{\sym{m}})
  \;=\;
  \min_{\theta \in \Theta}
  \Big\{\, L(\theta) \;-\; \log Q_\theta(y^{\sym{m}} \mid x^{\sym{m}}) \,\Big\}.
\end{equation}
\end{definition}
The first term acts as a complexity penalty: parameters that require more bits to describe must yield correspondingly larger reductions in the data term to be selected. Despite its conceptual clarity, the two-part code is generally suboptimal among universal codes, and designing $L(\theta)$ for high-dimensional models such as neural networks requires discretizing $\Theta$ to some precision, with the choice itself affecting the trade-off~\citep{grunwald2007minimum}. The prequential code, described next, sidesteps both issues by avoiding any explicit encoding of $\theta$.

\subsection{The prequential encoding scheme}
The prequential code encodes the data sequentially~\citep{dawid1984present, grunwald2007minimum}, using at each step a predictive distribution derived from the data seen so far; Alice and Bob together transmit the labels of the training set while implicitly transmitting the learned model.

Alice and Bob share, prior to any communication, the network architecture, the initialization $\theta_0$, the optimizer and its hyperparameters, and the unlabeled data $\{x_i\}_{i=1}^{\sym{m}}$. \add{Here $\theta \in \Theta$ are the network weights and $Q_\theta(y \mid x)$ is the network's predictive (softmax) distribution, with $Q_\theta(y^m \mid x^m) = \prod_{i} Q_\theta(y_i \mid x_i)$.} Let $B_1, \dots, B_{\sym{K}}$ be a partition of $\{1, \dots, \sym{m}\}$ into $\sym{K}$ consecutive blocks; the case $\sym{K} = \sym{m}$ recovers the classical one-by-one prequential code, while coarser partitions are typically used for computational efficiency~\citep{blier2018description}.

\begin{definition}[Prequential code length]
\label{def:prequential}
Let $\hat{\theta}_t$ be the model obtained by running the shared training procedure on $\bigcup_{k=1}^{t} B_k$, with $\hat{\theta}_0 = \theta_0$. The \emph{prequential code length} is
\begin{equation}
  \Lpreq(y^{\sym{m}} \mid x^{\sym{m}})
  \;=\;
  \sum_{t=0}^{\sym{K}-1}\;\sum_{i \in B_{t+1}}
  -\log Q_{\hat{\theta}_t}(y_i \mid x_i),
\end{equation}
i.e., the labels in block $B_{t+1}$ are encoded under the model trained on blocks $B_1, \dots, B_t$. \add{We write $Q_{\mathrm{preq}}(y^m \mid x^m) := \prod_{t=0}^{K-1} \prod_{i \in B_{t+1}} Q_{\hat{\theta}_t}(y_i \mid x_i)$ for the prequential predictive distribution, so that $\Lpreq = -\log Q_{\mathrm{preq}}(y^m \mid x^m)$.}
\end{definition}

No parameters are ever transmitted: Bob can replicate each $\hat{\theta}_t$ from the shared protocol and the previously decoded labels, so the model is communicated implicitly through the sequence of predictive distributions. Once Bob has received $y^{\sym{m}}$, he can reconstruct the parameters $\theta^*$ obtained by training on the full set $(x^{\sym{m}}, y^{\sym{m}})$, achieving the in-hindsight \new{code length} $\hat{L}(y^{\sym{m}} \sym{\mid x^{\sym{m}}}) = -\log Q_{\theta^*}(y^{\sym{m}} \mid x^{\sym{m}})$. The gap between the prequential code length and \new{this code length} is central to our analysis.

\begin{definition}[Regret]
\label{def:regret}
The \emph{regret} of the prequential code is $R := \Lpreq(y^{\sym{m}} \mid x^{\sym{m}}) - \hat{L}(y^{\sym{m}} \mid x^{\sym{m}}).$
\end{definition}

Intuitively, $R$ counts the extra bits incurred by encoding the parameters jointly with the data, and plays the role of a model-complexity penalty analogous to $L(\theta)$ in the two-part code.

Notably, in classical universal-coding settings with the maximum likelihood estimator (MLE), $R \geq 0$ holds pointwise~\citep{grunwald2007minimum} (to avoid 0 probability and hence infinite log, a small pseudo-count can be added to each class before any data is seen). In our setting where $\hat\theta_t$ is produced by a training procedure rather than an MLE predictor, $R \geq 0$ should hold with any reasonable training procedure, since the final model is trained on the full dataset and should fit the labels at least as well as the intermediate models trained on prefixes. Thus, we assume $R \geq 0$ throughout the theoretical analysis.

\section{Compression of Parameters and the Modern Information Bottleneck}
\label{sec:ib}

Compression and generalization of deep neural networks have been studied rarely through MDL, yet widely through the information bottleneck (IB). Before developing our MDL-based analysis, we therefore clarify that MDL does not use ``compression'' in an unrelated sense, and why it is a natural perspective closely related to IB, \new{even though the concepts} of compression in MDL and IB seem to be different. 

Specifically, the IB places compression at the level of representations, calling $Z = \phi_l(X)$ compressive when $I(X;Z)$ or $I(X;Z \mid Y)$ is small, while MDL (Section~\ref{section:mdl}) places it at the level of parameters, calling a model compressive when its weights admit a short code. We show that the two notions meet at a single quantity once the encoder is learned from the same sample whose representation it encodes. The in-sample $I(X^m; Z^m)$ decomposes into a fixed-encoder IB term plus the encoder-complexity correction that the modern IB requires~\citep{kawaguchi2023does}, and the correction is upper-bounded by an MDL parameter description length, the quantity we operationalize in Section~\ref{section:regret-as-model-complexity}.

\textbf{The in-sample IB and its decomposition.}
Let $(X^{\sym{m}}, Y^{\sym{m}}) \sim P$ be a training sample, and let a (possibly randomized) learning algorithm produce an encoder $\phi_l$\add{, the map from an input to its layer-$l$ representation,} via $T(\phi_l \mid X^{\sym{m}}, Y^{\sym{m}})$; for a deterministic algorithm this is a point mass. Write $T$ for the induced joint on $(X^{\sym{m}}, Y^{\sym{m}}, \phi_l)$, and $Z^{\sym{m}} = \phi_l(X^{\sym{m}})$ for the representations of the training inputs under the trained encoder.

The classical IB studies $I(X; Z)$ for an encoder fixed in advance. The quantity actually measured in practice, e.g., when IB-style mutual information is used as a regularizer in deep networks~\citep{alemi2016deep, achille2018information}, is $I(X^{\sym{m}}; Z^{\sym{m}})$ under the joint law $T$, in which $\phi_l$ depends on $X^{\sym{m}}$. The classical criterion is known to fail in this regime: \citet{hafez2020sample} exhibit an encoder that memorizes how $X$ maps to $Y$ on the training set yet still appears compressive by the classical criterion. The criterion ignores how $\phi_l$ depends on the sample, and the chain rule on $(X^{\sym{m}}; Z^{\sym{m}}, \phi_l)$ makes the missing term explicit. 
% The diagnosis is structural---the criterion is blind to how $\phi_l$ depends on the sample---and the chain rule on $(X^n; Z^n, \phi_l)$ exposes precisely which channel it overlooks.

\begin{proposition}
\label{prop:ib-decomposition}
Let $(X^{\sym{m}}, Y^{\sym{m}}, \phi_l) \sim T$ as above, and $Z^{\sym{m}}$ a measurable function of $(X^{\sym{m}}, \phi_l)$. Then
\begin{equation}
I(X^{\sym{m}}; Z^{\sym{m}}) \;=\; I(X^{\sym{m}}; Z^{\sym{m}} \mid \phi_l) \;+\; I(X^{\sym{m}}; \phi_l) \;-\; I(X^{\sym{m}}; \phi_l \mid Z^{\sym{m}}),
\label{eq:ib-identity}
\end{equation}
and consequently
\begin{equation}
I(X^{\sym{m}}; Z^{\sym{m}}) \;\le\; I(X^{\sym{m}}; Z^{\sym{m}} \mid \phi_l) \;+\; I(X^{\sym{m}}; \phi_l).
\label{eq:ib-inequality}
\end{equation}
\end{proposition}
All proofs are deferred to Appendix~\ref{app:proofs}.

We interpret the two terms as follows. 
The first term, $I(X^{\sym{m}}; Z^{\sym{m}} \mid \phi_l)$, is the fixed-encoder IB averaged over the algorithm-induced encoder law. It reduces to the classical IB exactly when $\phi_l$ is constant. The second term, $I(X^{\sym{m}}; \phi_l)$, is an encoder-complexity correction. It measures how much the trained encoder varies with the training data, and vanishes whenever the algorithm returns the same $\phi_l$ on every sample. 

\textbf{From encoder complexity to parameter description length.}
We next show that the encoder-complexity correction is bounded by an MDL parameter description length through a chain of elementary inequalities. Suppose $\phi_l$ is realized by a parameter vector $\theta$, and assume $\theta$ takes values in a discrete set; i.e., in practice, the quantization implicit in any finite-precision implementation. Then
\begin{equation}
I(X^{\sym{m}}; \phi_l) \;\le\; H(\phi_l) \;\le\; H(\theta) \;\le\; \mathbb{E}[L(\theta)],
\label{eq:ib-mdl-chain}
\end{equation}
where $L$ is a prefix-free code on the parameter space. The first inequality is the discrete-MI bound, the second is data processing as $\phi_l$ is a function of $\theta$, and the third is Shannon's source coding theorem.

A related encoder-complexity term $I(\phi_l; S)$\new{, in which $S$ is the training set and $\phi_l$ is the encoder learned from it, appears in \citet{kawaguchi2023does},} who add it to the classical IB bound to prove the first rigorous IB-based generalization guarantee. Our $I(X^{\sym{m}}; \phi_l)$ is bounded by $I(\phi_l; S)$, and the parameter description length in \eqref{eq:ib-mdl-chain} in turn bounds $I(\phi_l; S)$. The IB correction needed for learned encoders is bounded by a parameter description length, so a model with parameters admitting a short code in the MDL sense is also compressive in the modern IB sense. The remainder of the paper takes up this parameter-level quantity directly, showing how it can be measured from the training trajectory alone (Section~\ref{section:regret-as-model-complexity}) and how it translates into a generalization bound (Section~\ref{sec:pacmdl}).

\section{Prequential Regret Bounds the Description Length of Neural Networks} \label{section:regret-as-model-complexity}
This section formalizes the intuition that the prequential regret contains the description length of the trained network. We begin with the conceptual intuition, then introduce the two probability spaces needed to make it precise, and state the formal connection in Theorem~\ref{thm:preq}.

\textbf{The conceptual gap.} The intuition starts from how the prequential encoding scheme can be interpreted. Consider the message transmission between Alice and Bob again: the total cost $\Lpreq$ is the number of bits Alice spends, and by the time Bob has decoded the labels he can replay the training procedure and recover the same parameters Alice produced; he ends up holding both data and model. 

Thus, informally, one could argue that the description length of the model is intuitively \new{``contained''} in the prequential code, from which one can decode both the model and the data. 
\new{Splitting} this total into ``bits for the data'' and ``bits for the model'', the data part is well-approximated by $\hat L$, the in-hindsight code length under the trained model, and the model part is what remains, namely the regret $\Lpreq - \hat L$. 
Specifically, this intuition was articulated by~\citet{blier2018description} as motivation for using prequential codes to argue Occam's razor for neural networks, in which they argue that the description length of \new{the} model is `contained' in $\Lpreq$.

However, this conceptual intuition is heuristic, as the prequential code is literally a code on the data, not on the parameters: every bit of $\Lpreq$ is allocated to encoding a label, not to encoding $\theta$. Yet Bob recovers $\theta$ at the end, so the parameter information must be in those bits somewhere, just not in a designated parameter section of the code. Whether the heuristic holds rigorously, and what description-length quantity it formally \new{upper-bounds}, has been left open.

\textbf{Setup.}
Making the conceptual intuition precise requires two probability spaces: 1) the model's predictive distribution $Q_\theta$, which assigns likelihoods to data given parameters, and 2) the joint distribution $T$ over data and trained parameters, induced by the data-generating process and the algorithm $\cA$ together. 

We consider a training set $(X^{\sym{m}}, Y^{\sym{m}})$ drawn from $P$ and write $P(Y^{\sym{m}} \mid X^{\sym{m}})$ for the conditional data-generating distribution. A learning algorithm $\cA$ run on $(X^{\sym{m}}, Y^{\sym{m}})$ produces a parameter $\theta(X^{\sym{m}}, Y^{\sym{m}})$, random under the algorithm's internal randomness $U$; we write $T(\theta \mid X^{\sym{m}}, Y^{\sym{m}})$ for the resulting conditional distribution. Marginalizing over labels gives $T(\theta \mid X^{\sym{m}}) = \sym{\sum_{y^{\sym{m}}}} T(\theta \mid X^{\sym{m}}, \sym{y^{\sym{m}}})\, P(\sym{y^{\sym{m}}} \mid X^{\sym{m}})$, the distribution over trained parameters when the inputs are fixed and the labels (together with the algorithm's randomness) vary\add{; $T(Y^m \mid X^m, \theta)$ denotes the corresponding reverse conditional, the posterior over label sequences given the trained parameter}. We further consider the parametric family $\{Q_\theta\}_{\theta \in \Theta}$ defining the model and write $Q_{\mathrm{preq}}(\cdot \mid X^{\sym{m}})$ for the prequential predictive distribution induced by~$\cA$. 

Notably, $X^{\sym{m}}$ is treated as shared side information in the coding protocol throughout this section, \new{as is common practice} in supervised MDL~\citep{yang2022truly, proencca2020interpretable, yang2024conditional}, in which the goal is to learn $P(Y \mid X)$ rather than $P(X, Y)$. Thus, the cost of encoding $X^{\sym{m}}$ does not contribute to the code length attributable to learning, and we condition on $X^{\sym{m}}$ throughout. We keep this conditioning implicit when convenient. Section~\ref{sec:pacmdl} relaxes this in the PAC-MDL setting.

We next present our main theoretical results that the prequential regret can \new{upper-bound the} parameter description length. 

\begin{theorem}[Prequential regret \new{upper-bounds} parameter description length]
\label{thm:preq}
Let $\theta = \theta(X^{\sym{m}}, Y^{\sym{m}})$ denote the random parameter produced by training, $\Lfinal = -\log Q_{\theta}(Y^{\sym{m}} \mid X^{\sym{m}})$, and $\Lpreq = -\log Q_{\mathrm{preq}}(Y^{\sym{m}} \mid X^{\sym{m}})$. Suppose that, pointwise,
\begin{equation}
  Q_{\theta}(Y^{\sym{m}} \mid X^{\sym{m}}) \;\geq\; T(Y^{\sym{m}} \mid X^{\sym{m}}, \theta).
  \label{eq:preq-assumption}
\end{equation}
Then the following two statements hold.
\begin{enumerate}
\item[(i)] (Pointwise.) For every $(X^{\sym{m}}, Y^{\sym{m}})$,
\begin{equation}
  \Lpreq - \Lfinal
  \;\geq\;
  \underbrace{-\log T(\theta \mid X^{\sym{m}}) + \log T(\theta \mid X^{\sym{m}}, Y^{\sym{m}})}_{\text{Description length of }\theta \,\,\text{for fixed} \,\,X^{\sym{m}},\, Y^{\sym{m}}}
  \;+\; \log \frac{P(Y^{\sym{m}} \mid X^{\sym{m}})}{Q_{\mathrm{preq}}(Y^{\sym{m}} \mid X^{\sym{m}})}.
  \label{eq:preq-pointwise}
\end{equation}
\item[(ii)] (In expectation.) Taking expectations under the joint distribution over $(X^{\sym{m}}, Y^{\sym{m}}, \theta)$,
\begin{equation}
  \mathbb{E}[\Lpreq - \Lfinal]
  \;\geq\;
  \underbrace{H_T(\theta \mid X^{\sym{m}}) - H_T(\theta \mid X^{\sym{m}}, Y^{\sym{m}})}_{\text{Expected description length of }\theta}
  \;+\; \fix{\E_{X^{m}}\bigl[}\mathrm{KL}\bigl(P(\cdot \mid X^{\sym{m}}) \,\|\, Q_{\mathrm{preq}}(\cdot \mid X^{\sym{m}})\bigr)\fix{\bigr]},
  \label{eq:preq-expected}
\end{equation}
where both terms on the right are non-negative.
\end{enumerate}
\end{theorem}

The interpretation of the bound depends on whether training is deterministic or \new{random} given the data, as we describe in detail \new{below}. 

\emph{Deterministic training.} If the algorithm is fully deterministic given $(X^{\sym{m}}, Y^{\sym{m}})$, e.g., with a fixed seed controlling batch order, dropout, etc., then $\theta$ is a function of the training set, $T(\theta \mid X^{\sym{m}}, Y^{\sym{m}})$ is a point mass, and $H_T(\theta \mid X^{\sym{m}}, Y^{\sym{m}}) = 0$. Inequality~\eqref{eq:preq-expected} simplifies to
\begin{equation}
  \mathbb{E}[\Lpreq - \Lfinal]
  \;\geq\;
  H_T(\theta \mid X^{\sym{m}}) + \E_{X^{\sym{m}}}[\KL(P \| Q_{\mathrm{preq}})]
  \;\geq\;
  \mathbb{E}[-\log T(\theta \mid X^{\sym{m}})],
  \label{eq:preq-deterministic}
\end{equation}
i.e., the \fix{expected} model complexity $\fix{\E[}\sym{R}\fix{]} = \fix{\E[}\Lpreq - \Lfinal\fix{]}$ is at least the conditional self-information of the trained parameters under their marginal $T(\theta \mid X^{\sym{m}})$. Thus, this interpretation reveals our key finding that \emph{prequential regret \new{``behaves like''} a parameter code length}, with description length measured against the algorithm's induced distribution over trained parameters at the given inputs.

\emph{Stochastic training and the bits-back interpretation.} When training retains randomness (random seeds, stochastic batching, dropout, etc.), $H_T(\theta \mid X^{\sym{m}}, Y^{\sym{m}}) > 0$ and the residual $H_T(\theta \mid X^{\sym{m}}) - H_T(\theta \mid X^{\sym{m}}, Y^{\sym{m}}) = I_T(Y^{\sym{m}}; \theta \mid X^{\sym{m}})$ takes its place on the right-hand side. 

Notably, $I_T(Y^{\sym{m}}; \theta \mid X^{\sym{m}})$ can be interpreted with the \emph{bits-back} coding argument~\citep{hinton1993keeping, honkela2004variational}: when Alice transmits a stochastic parameter $\theta \sim T(\theta \mid X^{\sym{m}}, Y^{\sym{m}})$, she nominally pays $H_T(\theta \mid X^{\sym{m}})$ bits, but Bob, given $\theta$ and the protocol, can recover the random bits used to sample $\theta$ from $T(\theta \mid X^{\sym{m}}, Y^{\sym{m}})$ and reclaim $H_T(\theta \mid X^{\sym{m}}, Y^{\sym{m}})$ of them. The net description length of $\theta$ under this stochastic protocol is therefore $I_T(Y^{\sym{m}}; \theta \mid X^{\sym{m}})$, exactly the term appearing in~\eqref{eq:preq-expected}. Inequality~\eqref{eq:preq-expected} thus extends the deterministic statement to the stochastic setting. 

\begin{remark}
    The condition $Q_{\theta}(Y^{\sym{m}} | X^{\sym{m}}) \geq T(Y^{\sym{m}} | X^{\sym{m}}, \theta)$ \new{can be interpreted as follows}. \add{The LHS is the likelihood of the training labels under the trained model, which is large in the nearly interpolating regime; the RHS is the posterior probability of these labels among all label sequences that the algorithm could have mapped to $\theta$. The condition holds whenever the latter is the smaller of the two. This is the case in typical many-to-one training, where many label sequences are mapped to the same parameter, so that $\theta$ leaves substantial uncertainty about $Y^m$; it fails \fix{when the model underfits, so that the LHS is small, or} under strong memorization, where $\theta$ all but uniquely identifies $Y^m$ and pushes $T(Y^m \mid X^m, \theta)$ close to one.} We hence refer to this condition as the \emph{interpolation-genericity} condition. Direct verification is unavailable since $T(\theta \mid X^{\sym{m}}, Y^{\sym{m}})$ has no closed form for nontrivial deep learning algorithms, so we test the theorem's operational prediction in Section~\ref{section:exp}, i.e., $\Lpreq - \Lfinal$ tracks generalization.
\end{remark}

In summary, the intuition that the prequential code length contains a description of the trained model, used informally by~\citet{blier2018description}, becomes precise once two probability spaces are distinguished. The predictive distribution $Q_\theta$ governs how the trained model assigns likelihoods to data, while the algorithm-induced distribution $T(\theta \mid X^{\sym{m}})$ governs how the trained parameters themselves vary under the data-generating process and the training algorithm. The heuristic interpretation cannot say which description-length quantity is bounded by the regret, because it does not separate these two spaces. Theorem~\ref{thm:preq} shows that the relevant quantity is the description length of $\theta$ under $T(\theta \mid X^{\sym{m}})$, the algorithm's induced distribution over trained parameters at the given inputs.

\section{From Prequential Regret to Generalization}\label{sec:pacmdl}

The previous section established that the \fix{expected} prequential regret $\fix{\E[}\Lpreq - \hat L\fix{]}$ operationally upper-bounds the description length of the trained parameter under the algorithm's induced distribution $T(\theta \mid X^m)$. We now leverage this description length to prove a generalization guarantee. For deterministic algorithms we base our theoretical results on the PAC-MDL bound of \citet{blum2003pac}, which converts any sub-probability prior over description methods into a high-probability bound on test error; for stochastic algorithms we extend the argument to the transductive PAC-Bayes setting of \citet{begin2014pac}, which accommodates the bits-back argument as we identified in Section~\ref{section:regret-as-model-complexity}. We first recall the PAC-MDL bound in its standard form.

\begin{theorem}[PAC-MDL bound; \citet{blum2003pac}, Theorem 6]
\label{thm:pac-mdl}
Let $D$ be a distribution over $\cX \times \cY$. Draw $m+n$ fixed labeled examples $(x,y)^{m+n} \sim D^{m+n}$ and split them uniformly at random into a training set $S$ of size $m$ and a test set $S'$ of size $n$. A \emph{description method} $\sigma$ is a binary string that, together with the unlabeled data $x^{m+n}$, determines a labeling $y_\sigma$ of all $m+n$ examples; write $\hat{e}(\sigma, S)$ and $\hat{e}(\sigma, S')$ for the error rates of $y_\sigma$ on the training and test sets, respectively. Let $P(\sigma) \geq 0$ be any \emph{prior} over description methods satisfying $\sum_\sigma P(\sigma) \leq 1$. Then with probability at least $1 - \delta$ over the draw of $(x,y)^{m+n}$ and the random partition into $S, S'$:
\[
  \forall \sigma \quad
  n\hat{e}(\sigma, S')
  \;\leq\;
  b_{\max}\!\big(n,\;\hat{e}(\sigma, S),\;P(\sigma)\delta\big),
\]
where $b_{\max}(n, a, \beta)$ is the largest integer $b$ such that $\mathrm{Bucket}(m, n, a, b) \geq \beta$, with $\mathrm{Bucket}(m, n, a, b) = \sum_{t=b}^{a+b} \binom{n}{t}\binom{m}{a+b-t} \big/ \binom{m+n}{a+b}$.
\end{theorem}

\begin{remark}
Theorem~\ref{thm:pac-mdl} controls how much the test error $\hat e(\sigma, S')$ can exceed the training error $\hat e(\sigma, S)$ as a function of the description length $-\log P(\sigma)$. Shorter descriptions yield tighter bounds, since $P(\sigma)\delta$ is larger and the right-hand side is smaller.
\end{remark}

The key feature of Theorem~\ref{thm:pac-mdl} is that the prior $P(\sigma)$ is unconstrained beyond being a sub-probability, and can depend on objects fixed in advance of the random partition into $S, S'$, including the unlabeled sample $x^{m+n}$ shared between Alice and Bob. 

However, the algorithm-induced distribution $T(\theta \mid X^m)$ is partition-dependent, since $X^m$ is determined by which subset of $x^{m+n}$ ends up in the training partition. We next show that we resolve this challenge by treating the description method as a pair $(\pi, \theta)$ specifying both the partition and the trained parameter, with the partition encoded uniformly. As a result, the resulting joint prior is partition-invariant by construction.

\subsection{Deterministic algorithms}
We next apply the joint prior to the PAC-MDL \new{bound} and present the generalization bound for deterministic algorithms. 
\begin{theorem}[Generalization bound via prequential regret, deterministic algorithms]
\label{thm:generalization}
Let $\cA$ be a deterministic algorithm, and consider the PAC-MDL setup of Theorem~\ref{thm:pac-mdl}: draw $m+n$ labeled examples $(x, y)^{m+n} \sim P^{\otimes(m+n)}$ and split them uniformly at random into a training set $S$ of size $m$ and a test set $S'$ of size $n$. Let $\theta = \cA(S)$ denote the parameter produced by training on $S$, with induced conditional distribution $T(\theta \mid X^m)$ over $\theta$ given the training inputs. Suppose the interpolation-genericity condition~\eqref{eq:preq-assumption} holds and $\theta$ takes values in a discrete (finite-precision) parameter space. For any $\delta \in (0, 1)$, with probability at least $1 - \delta$ over the joint draw of the sample and the partition,
\begin{equation}
  n \hat e(\theta, S')
  \;\leq\;
  b_{\max}\!\left(n,\; \hat e(\theta, S),\; \frac{\delta}{2 \binom{m+n}{m}} \cdot 2^{-2\E[\Lpreq - \hat L]/\delta}\right),
  \label{eq:generalization-pac-mdl}
\end{equation}
where $L_{\mathrm{preq}}$ and $\hat L$ are the prequential and in-hindsight code lengths of Section~\ref{section:mdl} evaluated on the training partition. $\mathbb{E}[L_{\mathrm{preq}} - \hat L]$ is taken over the joint draw of $(x,y)^{m+n}$ and the random partition.
\end{theorem}

\begin{remark}
The proof (Appendix~\ref{app:proofs}) combines a PAC-MDL bound at confidence $1 - \delta/2$, using the partition-invariant prior $P(\pi, \theta) = T(\theta | X^m_\pi)/\binom{m+n}{m}$, together with a Markov inequality on the description length at confidence $1 - \delta/2$. The key step is the expectation chain $\E[-\log T(\theta | X^m_{\pi^*})] = H_T(\theta | X^m) \leq \E[\Lpreq - \hat L]$, used in the Markov step: the algorithm-dependent factor in the bound is exactly the prequential regret, which can be computed directly from the training set.
\end{remark}

\subsection{Stochastic algorithms}

Theorem~\ref{thm:generalization} assumes $\cA$ is deterministic, so $T(\theta \mid X^m, Y^m)$ is a Dirac mass and the description length $-\log T(\theta^* \mid X^m_{\pi^*})$ of the realized parameter \new{is bounded by} the prequential regret in expectation. For stochastic $\cA$ \new{this bound} no longer holds, as the posterior $T(\theta \mid X^m, Y^m)$ is non-degenerate, and
\[
  \E\bigl[-\log T(\theta^* \mid X^m)\bigr] \;=\; H_T(\theta \mid X^m) \;=\; I_T(\theta; Y^m \mid X^m) + H_T(\theta \mid X^m, Y^m),
\]
which \new{can exceed} $\E[\Lpreq - \hat L]$ by \new{up to} the algorithm's intrinsic conditional entropy $H_T(\theta \mid X^m, Y^m)$. 

Therefore, applying PAC-MDL via Markov as in Theorem~\ref{thm:generalization} yields a bound with irreducible slack: the framework charges $-\log P(\sigma)$ at the realized $\sigma$ pointwise, and does not account for the bits-back term that prequential coding extracts in Section~\ref{section:regret-as-model-complexity}. 

We close the gap by replacing PAC-MDL with its PAC-Bayes counterpart \citep{begin2014pac} in the same transductive setup. With $T(\theta \mid X^m, Y^m)$ playing the role of the data-dependent posterior and $T(\theta \mid X^m)$ the data-independent prior, the complexity term becomes $\KL\bigl(T(\theta \mid X^m, Y^m) \,\|\, T(\theta \mid X^m)\bigr)$, whose expectation is then $I_T(\theta; Y^m \mid X^m)$, i.e., the bits-back-tight quantity. 

The cost is that the bound now controls the Gibbs test error $\E_{\theta \sim T(\cdot \mid X^m, Y^m)}[\hat e(\theta, S \cup S')]$ rather than the error of a single sampled $\theta^*$, which is the natural object for a stochastic predictor and reduces to the point error in the deterministic limit.

\begin{theorem}[Generalization bound via prequential regret, stochastic algorithms]
\label{thm:generalization-stochastic}
Let $\cA$ be a (possibly stochastic) learning algorithm with induced conditional distribution $T(\theta \mid X^m, Y^m)$. Consider the PAC-MDL setup of Theorem~\ref{thm:pac-mdl}: draw $m+n$ labeled examples $(x, y)^{m+n} \sim P^{\otimes(m+n)}$ and split uniformly at random into a training set $S$ of size $m$ and a test set $S'$ of size $n$. Suppose the interpolation-genericity condition~\eqref{eq:preq-assumption} holds and $\theta$ takes values in a discrete (finite-precision) parameter space. For any $\delta \in (0, 1)$, with probability at least $1 - \delta$ over the joint draw of sample and partition,
\begin{equation}
  \mathrm{kl}\!\Bigl(\sym{\hat e(\rho^*, S)} \,\Big\|\, \sym{\hat e(\rho^*, S \cup S')}\Bigr)
  \;\leq\;
  \frac{1}{m}\!\left[\log\binom{m+n}{m} + \frac{2\,\E[\Lpreq - \hat L]}{\delta} + \log\frac{2\,\xi(m, n)}{\delta}\right],
  \label{eq:generalization-stochastic-hp}
\end{equation}
where $\rho^* = T(\cdot \mid X^m, Y^m)$, $\sym{\hat e(\rho^*, D)} = \E_{\theta \sim \rho^*}[\hat e(\theta, D)]$ for $D \in \{S, S \cup S'\}$, and $\xi(m, n)$ is the transductive PAC-Bayes concentration constant of \citet{begin2014pac}.
As in Theorem~\ref{thm:generalization}, $L_{\mathrm{preq}}$ and $\hat L$ are evaluated on the training partition $S$, and the expectation $\mathbb{E}[L_{\mathrm{preq}} - \hat L]$ is taken over the joint draw of $(x,y)^{m+n}$ and the random partition.
\end{theorem}

\begin{remark}
The key step is again an expectation chain, now in mutual-information form: $\E\bigl[\KL\bigl(T(\theta \mid X^m, Y^m) \,\|\, T(\theta \mid X^m)\bigr)\bigr] = I_T(\theta; Y^m \mid X^m) \leq \E[\Lpreq - \hat L]$. This sharpens the deterministic chain $H_T(\theta \mid X^m) \leq \E[\Lpreq - \hat L]$ by exactly the algorithm's intrinsic entropy $H_T(\theta \mid X^m, Y^m)$\new{, i.e., the bits-back term}. Two technical points: \eqref{eq:generalization-stochastic-hp} is stated on $S \cup S'$ following the convention of \citet{begin2014pac}, and is related to test-error-only forms via $\sym{\hat e(\rho^*, S \cup S') = \bigl(m\, \hat e(\rho^*, S) + n\, \hat e(\rho^*, S')\bigr)/(m+n)}$; and the discrete-parameter assumption inherited from Theorem~\ref{thm:generalization} ensures that all KL divergences are well-defined as sums, with the continuous case requiring only $T(\cdot \mid X^m, Y^m) \ll T(\cdot \mid X^m)$ for the relative entropy to be finite.
\end{remark}

\section{Experiments}\label{section:exp}
We empirically test whether the prequential regret $L_{\mathrm{preq}}-\hat L$ tracks the generalization gap on CIFAR-10 and CIFAR-100~\citep{Krizhevsky09learningmultiple}, sweeping over a range of network configurations and recording the regret alongside the held-out generalization error.

\textbf{Setup.}
We compute $L_{\mathrm{preq}}$ following the empirical prequential coding scheme of~\citet{blier2018description}, partitioning the training trajectory into $\sym{K}=40$ blocks of geometrically growing size, and obtain $\hat L$ from the final trained network. We trained 540 networks on CIFAR-10 and 99 networks on CIFAR-100 with ResNet architectures~\citep{he2016identity, he2016deep} of varying depth, sweeping over batch size, weight decay, dropout rate, and seed. An independent data subsample is drawn per configuration to enable comparison with the IB-based baselines~\citep{kawaguchi2023does}; our quantity itself does not require subsampling. Hyperparameters, schedules, and prequential blocking details are deferred to Appendix~\ref{appendix:exp}. Our code is available at \fix{\url{https://github.com/minimumdescriptionlength/mdl_gen}}.

\add{\textbf{Computing the regret.} Let $0 = m_0 < m_1 < \dots < m_K = m$ be the block boundaries, so that block $B_{t+1}$ consists of examples $m_t+1, \dots, m_{t+1}$. We next describe the procedure of computing the regret, in the following four steps. (1) The first block is encoded with the uniform distribution over classes, at $m_1 \log |\cY|$ bits. (2) For $t = 1, \dots, K-1$, $\hat{\theta}_t$ is trained on the first $m_t$ examples, warm-started from $\hat{\theta}_{t-1}$, and the cross-entropy of block $B_{t+1}$ under $Q_{\hat{\theta}_t}$ is added to $\Lpreq$. (3) $\hat L$ is the cross-entropy of the final model $\theta^*$, trained on all $m$ examples, on the \emph{training} set. (4) $R = \Lpreq - \hat L$. }

\begin{table}[ht]
\centering
\caption{Correlations between generalization-bound quantities and the empirical generalization gap on CIFAR-10, measured both in loss and in classification error. Notation follows the previous work \new{on IB generalization bounds}~\citep{kawaguchi2023does}. We write $S$ for the training sample, $D$ for the network depth, $\check I$ for a finite-sample mutual-information estimator, $Z^s_l$ for the activations at layer $l$ on data subsample $s$, $\theta^S_l$ for the parameters at layer $l$ trained on $S$, and $\theta^S_{D+1}$ for the full network parameter trained on $S$. The baseline $\check I(S; \theta^S_{D+1})$ is the parameter-information quantity $I(S; W)$ of the information-theoretic line~\citep{russo2016controlling, xu2017information} and is the only pure model-complexity baseline besides ours. Best results are bolded.}
\resizebox{0.85\textwidth}{!}{%
\begin{tabular}{lcccccc}
\toprule
& \multicolumn{2}{c}{Spearman corr.}
& \multicolumn{2}{c}{Pearson corr.}
& \multicolumn{2}{c}{Kendall corr.} \\
\cmidrule(lr){2-3} \cmidrule(lr){4-5} \cmidrule(lr){6-7}
Generalization gap:
& Loss & Error & Loss & Error & Loss & Error \\
\midrule
$\frac{1}{D}\sum_{l=1}^{D} \check{I}(X; Z_l^s)$
& 0.8481 & 0.7410 & 0.2116 & 0.1831 & 0.6425 & 0.5436 \\
$\min_{l\in[D]} \check{I}(X; Z_l^s)$
& 0.7145 & 0.5602 & 0.7203 & 0.5719 & 0.4461 & 0.3404 \\
$\frac{1}{D}\sum_{l=1}^{D} \check{I}(X; Z_l^s \mid Y)$
& 0.8481 & 0.7406 & 0.2140 & 0.1853 & 0.6427 & 0.5435 \\
$\min_{l\in[D]} \check{I}(X; Z_l^s \mid Y)$
& 0.7004 & 0.5434 & 0.7062 & 0.5560 & 0.4386 & 0.3305 \\
$\check{I}(S; \theta_{D+1}^{S})$
& 0.4688 & 0.3112 & 0.2512 & 0.0775 & 0.2121 & 0.1208 \\
$\min_{l\in[D]} \check{I}(S; \theta_l^{S}) + \check{I}(X; Z_l^s \mid Y)$
& 0.8434 & 0.7313 & 0.8437 & 0.7195 & 0.6270 & 0.5332 \\
\midrule
$\Lpreq - \hat L$ \textbf{(Ours)} & \textbf{0.8784} & \textbf{0.8846}
& \textbf{0.9211} & \textbf{0.9172}
& \textbf{0.6800} & \textbf{0.6874} \\
\bottomrule
\end{tabular}
}
\label{tab:cifar10_corr}
\end{table}

\textbf{Results on CIFAR-10.} Our hyperparameter and seed grid matches that of the IB-based baselines~\citep{kawaguchi2023does} on CIFAR-10, so we take their reported numbers directly into Table~\ref{tab:cifar10_corr}. In their pipeline, the mutual information terms $\check I(X;Z^s_l)$ and $\check I(X;Z^s_l\mid Y)$ are estimated by adaptive kernel density estimation~\citep{chelombiev2019adaptive} on layer-$l$ activations, while the parameter-information term $\check I(S;\theta^S_l)$ is estimated by fitting a \fix{SWAG~\citep{maddox2019simple}} posterior over parameters across data subsamples and applying a Jensen upper bound. The `Loss' and `Error' columns report the correlation between each complexity quantity and the generalization gap measured in loss and in classification error, respectively.

The prequential regret outperforms all baselines across all six rank and linear correlation columns. The most informative comparison is with $\check I(S;\theta^S_{D+1})$, the only other pure model-complexity quantity in the table (no representation-level term). The gap there is substantial: Pearson $0.92$ versus \new{$0.25$} on the loss gap, and $0.92$ versus \new{$0.08$} on the error gap. We attribute the gap to the estimation pipeline: $\check I(S;\theta^S_{D+1})$ requires fitting a high-dimensional Gaussian posterior over parameters and is sensitive to the SWAG approximation, whereas the prequential regret is read directly off the training trajectory and avoids high-dimensional density estimation altogether.

\begin{figure}[ht]
    \centering
    \includegraphics[width=.8\textwidth]{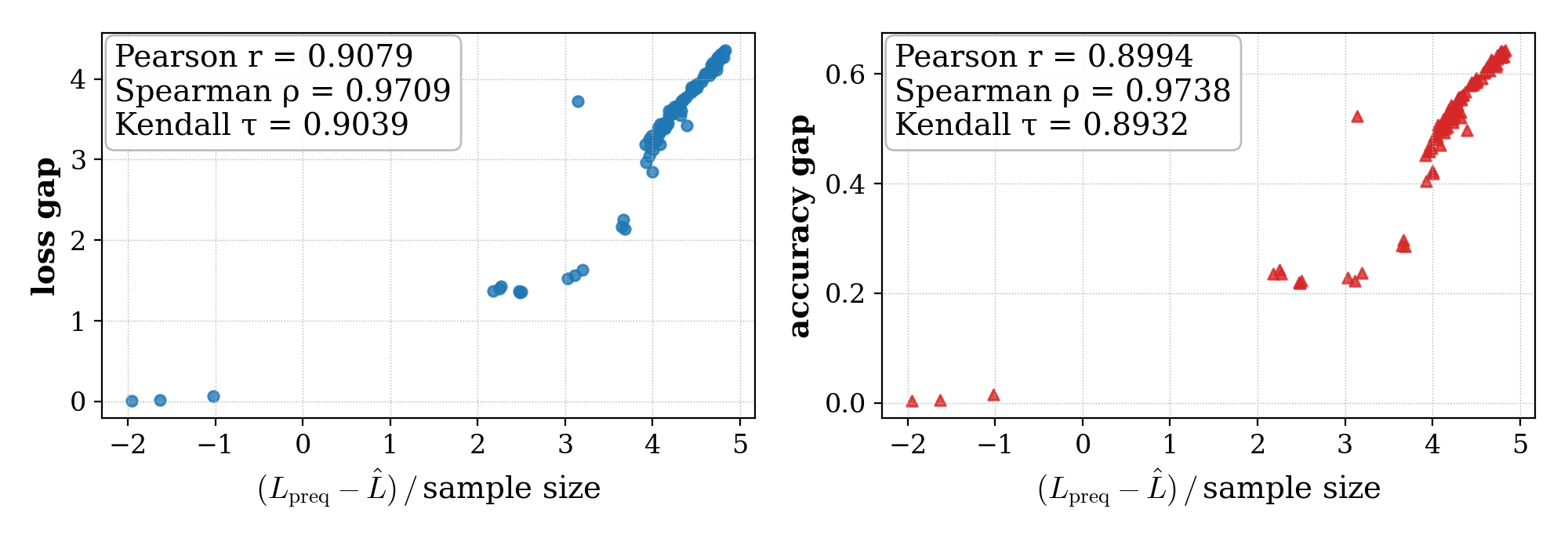}
    \caption{Correlation between the prequential regret and the generalization gaps on CIFAR-100.}
    \label{fig:lmodel_vs_gaps}
\end{figure}

\textbf{Results on CIFAR-100.} Figure~\ref{fig:lmodel_vs_gaps} plots the empirical prequential regret against the loss and accuracy gaps, and shows similarly strong correlations. The three outliers in the bottom-left quadrant have negative empirical regret and therefore cannot upper-bound any description length. They correspond to severely underfit configurations (training accuracy below $0.25$) where small networks are paired with overly aggressive weight decay, which violates the interpolation-genericity condition~\eqref{eq:preq-assumption} and places these cases outside the regime where Theorems~\ref{thm:generalization} and~\ref{thm:generalization-stochastic} apply. On the remaining configurations, which cover the regime of practical interest, the bound is well-defined and tracks the gap closely.

Due to the higher per-configuration cost on CIFAR-100, we restrict this evaluation to 99 network configurations. We do not report the IB baselines of~\citet{kawaguchi2023does} on CIFAR-100, as their post-training estimation did not finish within our 48-hour compute budget.

\section{Conclusion}\label{section:conclusion}
We approached the question of why over-parameterized neural networks generalize through the lens of MDL, and proposed the prequential regret $L_{\mathrm{preq}}-\hat L$ as an operational measure of model complexity computable from the training trajectory alone. The \fix{expected} regret upper-bounds the description length of the trained parameter under the algorithm-induced distribution, yields a generalization bound in the PAC-MDL setting, and tracks the empirical generalization gap on CIFAR-10 and CIFAR-100 with Pearson correlations exceeding $0.9$. Along the way, we connected parameter compression in the MDL sense to the algorithm-level information bottleneck, \new{showing that} the encoder-complexity correction recently added to IB \new{is upper-bounded by} an MDL parameter description length.

\textbf{Limitations and future work.} Our generalization bound holds in the transductive PAC-MDL setting, where the test set is finite and fixed. Extending it to a population-risk statement is a concrete open problem. 
Further, the prequential regret itself is architecture- and modality-agnostic, and tracking how it behaves on transformers and on language or multimodal data is an interesting future direction, particularly in regimes where the \new{nearly interpolating} assumption is stressed, such as large-scale pretraining.

\bibliography{reference.bib}

%%%%%%%%%%%%%%%%%%%%%%%%%%%%%%%%%%%%%%%%%%%%%%%%%%%%%%%%%%%%

\appendix

% \section{Technical appendices and supplementary material}
\section{Proofs}\label{app:proofs}

\subsection{Proof of Proposition~\ref{prop:ib-decomposition}} \label{app:proof:prop1}
\begin{proof}
Apply the chain rule to $(X^{\sym{m}}; Z^{\sym{m}}, \phi_l)$ in two ways:
\[
I(X^{\sym{m}}; Z^{\sym{m}}, \phi_l) \;=\; I(X^{\sym{m}}; \phi_l) + I(X^{\sym{m}}; Z^{\sym{m}} \mid \phi_l) \;=\; I(X^{\sym{m}}; Z^{\sym{m}}) + I(X^{\sym{m}}; \phi_l \mid Z^{\sym{m}}).
\]
Equating and rearranging gives \eqref{eq:ib-identity}; dropping $I(X^{\sym{m}}; \phi_l \mid Z^{\sym{m}}) \ge 0$ gives \eqref{eq:ib-inequality}.
\end{proof}

\subsection{Proof of Theorem~\ref{thm:preq}} \label{app:proof_thm_preq}
\begin{proof}
\emph{(i) Pointwise.} Taking $\log$ on both sides of~\eqref{eq:preq-assumption} and subtracting $\log Q_{\mathrm{preq}}(Y^{\sym{m}} \mid X^{\sym{m}})$,
\begin{equation}
  \log \frac{Q_{\theta}(Y^{\sym{m}} \mid X^{\sym{m}})}{Q_{\mathrm{preq}}(Y^{\sym{m}} \mid X^{\sym{m}})}
  \;\geq\;
  \log \frac{T(Y^{\sym{m}} \mid X^{\sym{m}}, \theta)}{Q_{\mathrm{preq}}(Y^{\sym{m}} \mid X^{\sym{m}})}.
  \label{eq:preq-step1}
\end{equation}
The LHS equals $\Lpreq - \Lfinal$. For the RHS, multiply and divide by $P(Y^{\sym{m}} \mid X^{\sym{m}})$:
\[
  \log \frac{T(Y^{\sym{m}} \mid X^{\sym{m}}, \theta)}{Q_{\mathrm{preq}}(Y^{\sym{m}} \mid X^{\sym{m}})}
  \;=\;
  \log \frac{T(Y^{\sym{m}} \mid X^{\sym{m}}, \theta)}{P(Y^{\sym{m}} \mid X^{\sym{m}})}
  \;+\;
  \log \frac{P(Y^{\sym{m}} \mid X^{\sym{m}})}{Q_{\mathrm{preq}}(Y^{\sym{m}} \mid X^{\sym{m}})}.
\]
% By Bayes' rule applied to the joint $T(Y^n, \theta \mid X^n) = P(Y^n \mid X^n) T(\theta \mid X^n, Y^n) = T(\theta \mid X^n) T(Y^n \mid X^n, \theta)$,
Since the algorithm-induced joint $T(\theta, Y^{\sym{m}} \mid X^{\sym{m}})$ factors two ways,
\[
  T(\theta \mid X^{\sym{m}})\, T(Y^{\sym{m}} \mid X^{\sym{m}}, \theta)
  \;=\;
  P(Y^{\sym{m}} \mid X^{\sym{m}})\, T(\theta \mid X^{\sym{m}}, Y^{\sym{m}}),
\]
so
\[
  \log \frac{T(Y^{\sym{m}} \mid X^{\sym{m}}, \theta)}{P(Y^{\sym{m}} \mid X^{\sym{m}})}
  \;=\;
  \log T(\theta \mid X^{\sym{m}}, Y^{\sym{m}}) - \log T(\theta \mid X^{\sym{m}}).
\]
Substituting back into~\eqref{eq:preq-step1} yields~\eqref{eq:preq-pointwise}.

\emph{(ii) In expectation.} Taking expectations of~\eqref{eq:preq-pointwise} over the joint distribution and using $\mathbb{E}[-\log T(\theta \mid X^{\sym{m}})] = H_T(\theta \mid X^{\sym{m}})$, $\mathbb{E}[-\log T(\theta \mid X^{\sym{m}}, Y^{\sym{m}})] = H_T(\theta \mid X^{\sym{m}}, Y^{\sym{m}})$, and $\mathbb{E}[\log P(Y^{\sym{m}} \mid X^{\sym{m}})/Q_{\mathrm{preq}}(Y^{\sym{m}} \mid X^{\sym{m}})] = \E_{X^{\sym{m}}}[\KL(P(\cdot \mid X^{\sym{m}}) \| Q_{\mathrm{preq}}(\cdot \mid X^{\sym{m}}))] \geq 0$, we obtain~\eqref{eq:preq-expected}.
\end{proof}

\subsection{Proof of Theorem~\ref{thm:generalization}} \label{app:prrof:thm_generalization_pac_mdl}

The proof has two steps: a PAC-MDL bound at confidence $1 - \delta/2$ with a partition-invariant prior, and a Markov bound at confidence $1 - \delta/2$ converting the expected-description-length result of Theorem~\ref{thm:preq} into a high-probability lower bound on $T(\theta \mid X^m)$. A union bound combines them.

\emph{Step 1: A partition-invariant prior and PAC-MDL.} A description method $\sigma = (\pi, \theta)$ specifies a training partition $\pi \subset [m+n]$ with $|\pi| = m$ together with a parameter $\theta$; given $x^{m+n}$, the labeling is determined by $y_\sigma(x_i) = \theta(x_i)$. Define
\[
  P(\pi, \theta) \;=\; \frac{T(\theta \mid X^m_\pi)}{\binom{m+n}{m}},
\]
where $X^m_\pi$ are the inputs at indices in $\pi$. The prior is determined by $\cA$ and $x^{m+n}$, both fixed in advance of the random partition, and sums to one since $\sum_\theta T(\theta \mid X^m_\pi) = 1$ for each $\pi$ and there are $\binom{m+n}{m}$ partitions. Applying Theorem~\ref{thm:pac-mdl} with confidence $1 - \delta/2$ gives that, with probability at least $1 - \delta/2$ over the joint draw, the realized $\sigma^* = (\pi^*, \theta^*)$ with $\theta^* = \cA(S)$ satisfies
\begin{equation}
  n \hat e(\theta^*, S') \;\leq\; b_{\max}\!\left(n,\, \hat e(\theta^*, S),\, \frac{T(\theta^* \mid X^m_{\pi^*})}{\binom{m+n}{m}} \cdot \frac{\delta}{2}\right),
  \label{eq:pac-mdl-applied}
\end{equation}
where we used that the labeling produced by $\sigma^*$ on $S$ and $S'$ coincides with that of $\theta^*$.

\emph{Step 2: Markov bound on the description length.} The random variable $-\log T(\theta \mid X^m)$ is non-negative, and Theorem~\ref{thm:preq}(ii) under the deterministic assumption ($H_T(\theta \mid X^m, Y^m) = 0$) gives $\E[-\log T(\theta \mid X^m)] = H_T(\theta \mid X^m) \leq \E[\Lpreq - \hat L]$. By Markov's inequality, with probability at least $1 - \delta/2$,
\begin{equation}
  -\log T(\theta \mid X^m) \;\leq\; \frac{2\, \E[\Lpreq - \hat L]}{\delta},
  \quad \text{i.e.,} \quad
  T(\theta \mid X^m) \;\geq\; 2^{-2\E[\Lpreq - \hat L]/\delta}.
  \label{eq:markov-T}
\end{equation}

\emph{Step 3: Union bound.} The events in \eqref{eq:pac-mdl-applied} and \eqref{eq:markov-T} each hold with probability at least $1 - \delta/2$ over the joint draw of sample and partition. By the union bound, both hold simultaneously with probability at least $1 - \delta$. On this event, since $b_{\max}(n, a, \cdot)$ is non-increasing in its third argument~\citep{blum2003pac}, plugging the lower bound on $T(\theta^* \mid X^m_{\pi^*})$ from \eqref{eq:markov-T} into \eqref{eq:pac-mdl-applied} gives
\[
  n \hat e(\theta^*, S')
  \;\leq\;
  b_{\max}\!\left(n,\, \hat e(\theta^*, S),\, \frac{\delta}{2 \binom{m+n}{m}} \cdot 2^{-2\E[\Lpreq - \hat L]/\delta}\right).
  \qed
\]

\subsection{Proof of Theorem~\ref{thm:generalization-stochastic}} \label{app:prrof:thm_generalization_pac_bayes}

The proof parallels that of Theorem~\ref{thm:generalization}, with transductive PAC-Bayes replacing PAC-MDL and Markov applied to a KL divergence rather than to a description length.

\emph{Step 1: Lifted prior and posterior; transductive PAC-Bayes.} Lift the hypothesis to a pair $(\pi, \theta)$ as in the proof of Theorem~\ref{thm:generalization}, and define
\[
  \pi_0(\pi, \theta) \;=\; \frac{T(\theta \mid X^m_\pi)}{\binom{m+n}{m}},
  \qquad
  \rho(\pi, \theta) \;=\; \indicator[\pi = \pi^*]\, T(\theta \mid X^m_{\pi^*}, Y^m_{\pi^*}),
\]
where $\pi^*$ is the realized partition. The prior $\pi_0$ is determined by $\cA$ and $x^{m+n}$, both fixed in advance of the partition, and is a probability distribution by the same calculation as in Theorem~\ref{thm:generalization}. Applying the transductive PAC-Bayes inequality (\citealp{begin2014pac}, Theorem 1) on the lifted hypothesis space at confidence $1 - \delta/2$ gives, with probability at least $1 - \delta/2$ over the joint draw of sample and partition,
\begin{equation}
  \mathrm{kl}\!\bigl(\sym{\hat e(\rho, S)} \,\|\, \sym{\hat e(\rho, S \cup S')}\bigr)
  \;\leq\;
  \frac{\KL(\rho \,\|\, \pi_0) + \log(2\xi(m, n)/\delta)}{m}.
  \label{eq:pac-bayes-applied}
\end{equation}
Since $\rho$ concentrates on $\pi = \pi^*$, the lifted Gibbs errors reduce to $\sym{\hat e(\rho, D) = \hat e(\rho^*, D)}$ for $\sym{D} \in \{S, S \cup S'\}$, and the KL decomposes as
\[
  \KL(\rho \,\|\, \pi_0)
  \;=\; \sum_\theta T(\theta \mid X^m_{\pi^*}, Y^m_{\pi^*}) \log \frac{T(\theta \mid X^m_{\pi^*}, Y^m_{\pi^*})}{T(\theta \mid X^m_{\pi^*}) / \binom{m+n}{m}}
\]
\[
\;=\; \log\!\binom{m+n}{m} + \KL\!\bigl(T(\theta \mid X^m_{\pi^*}, Y^m_{\pi^*}) \,\|\, T(\theta \mid X^m_{\pi^*})\bigr).
\]
\emph{Step 2: Markov on the KL.} The random variable $\KL\bigl(T(\theta \mid X^m, Y^m) \,\|\, T(\theta \mid X^m)\bigr)$ is non-negative, and by the identity $\E_{(B, C)}[\KL(P_{A \mid B, C} \,\|\, P_{A \mid C})] = I(A; B \mid C)$ applied with $A = \theta$, $B = Y^m$, $C = X^m$,
\[
  \E\bigl[\KL\bigl(T(\theta \mid X^m, Y^m) \,\|\, T(\theta \mid X^m)\bigr)\bigr] \;=\; I_T(\theta; Y^m \mid X^m).
\]
The bits-back extension of Theorem~\ref{thm:preq}(ii) gives $I_T(\theta; Y^m \mid X^m) \leq \E[\Lpreq - \hat L]$, so by Markov's inequality, with probability at least $1 - \delta/2$,
\begin{equation}
  \KL\bigl(T(\theta \mid X^m, Y^m) \,\|\, T(\theta \mid X^m)\bigr) \;\leq\; \frac{2\,\E[\Lpreq - \hat L]}{\delta}.
  \label{eq:markov-KL}
\end{equation}

\emph{Step 3: Union bound.} The events in \eqref{eq:pac-bayes-applied} and \eqref{eq:markov-KL} each hold with probability at least $1 - \delta/2$, so both hold simultaneously with probability at least $1 - \delta$. On this event, substituting \eqref{eq:markov-KL} into \eqref{eq:pac-bayes-applied} via the KL decomposition yields \eqref{eq:generalization-stochastic-hp}. \qed

\section{Additional Experiment Details}\label{appendix:exp}

\subsection{CIFAR-10 Experiment Details}\label{app:cifar10_experiments_details}
We replicate the 30\%-subsample CIFAR-10 protocol of \citet{kawaguchi2023does}. For each (seed, instance) pair we draw a 15{,}000-image subset of the CIFAR-10 training split using the original RNG order from their released codebase \new{(}a sequential \texttt{np.random.choice} over five instances after seeding\new{),} so our subsamples match theirs bit-for-bit. Inputs are normalized with the standard CIFAR-10 statistics; training uses random crops (\texttt{padding}$=4$) and horizontal flips, while bit-coded evaluation passes use no augmentation. We use the three architectures from the original work, PreResNet-$\{56, 83, 110\}$, imported directly as a submodule of their released code.

Each block is trained with SGD (momentum $=0.9$, initial learning rate $=0.1$, mean-reduction cross-entropy, no mixed precision, no label smoothing). The learning rate follows the piecewise schedule of \citet{kawaguchi2023does} with $\mathrm{swa\_start} = 0.805 \times E$ epochs: constant at the initial rate for the first half of pre-SWA training, linearly decayed to $\mathrm{swa\_lr} = 0.05$ over the next $40\%$, and constant at $\mathrm{swa\_lr}$ thereafter. We sweep weight decay $\in \{10^{-5}, 10^{-4}, 10^{-3}\}$ and batch size $\in \{64, 128, 1024\}$, identical to their Table 2 grid.

\textbf{Prequential MDL layer.} On top of this recipe we run the prequential coding procedure: each 15k-sample instance is partitioned into $\fix{K} \le 40$ blocks whose cumulative prefix sizes grow geometrically from $\max(\lceil 0.005\,\sym{m}\rceil, 100)$ to $\sym{m}$. \fix{The first block $B_1$} is coded with the uniform class prior ($\log_2 10$ bits/sample), and for each subsequent block we train on the prefix and evaluate cross-entropy on the held-out block. Per-block epoch budgets are tiered by prefix fraction $r$: $40/80/120/200$ epochs for $r < 0.05 / < 0.20 / < 0.50 / \ge 0.50$, so that the final block matches the original 200-epoch straight-training budget. We use warm-start (the model is carried across blocks) and report $\sym{R/m} = (L_{\mathrm{preq}} - \hat L)/\sym{m}$ as the predictor; targets are $\mathrm{GL} = \mathrm{test\_loss} - \mathrm{train\_loss}$ \fix{(in bits)} and $\mathrm{GE} = \mathrm{train\_acc} - \mathrm{test\_acc}$.

\textbf{Setup grid.} The full grid is $3\,(\mathrm{arch}) \times 3\,(\mathrm{wd}) \times 3\,(\mathrm{bs}) \times 4\,(\mathrm{seed}) \times 5\,(\mathrm{instance}) = 540$ cells, matching the CIFAR-10 grid of \citet{kawaguchi2023does}. Following their filter, correlations are reported on cells with training accuracy above $0.85$.

\subsection{CIFAR-100 Experiment Details}\label{app:cifar100_experiments_details}
We use the standard CIFAR-100 training split ($\sym{m} = 50{,}000$, no subsampling) and the held-out test split ($\sym{n} = 10{,}000$). Inputs are normalized with the standard CIFAR-100 statistics ($\mathrm{mean} = (0.5071, 0.4865, 0.4409)$, $\mathrm{std} = (0.2673, 0.2564, 0.2762)$); training uses random crops (\texttt{padding}$=4$) and horizontal flips, while bit-coded evaluation passes use no augmentation. Architectures are the torchvision ResNet family at five depths, ResNet-$\{18, 34, 50, 101, 152\}$, sweeping model size as a clean log-of-parameter-count axis without changing the residual block design.

Each block is trained with SGD (Nesterov momentum $=0.9$, initial learning rate $=0.1$, label smoothing $=0.1$, batch size $=256$, mean-reduction cross-entropy, no mixed precision). The learning-rate schedule is per-batch warmup-then-cosine: a linear warmup over the first $5\%$ of the block's optimizer steps starting at $0.1\times$ the initial rate, followed by cosine annealing to $\eta_{\min} = 10^{-4}$ over the remaining steps; the schedule is rebuilt at each block over its own step budget. We run three independent repeats per cell, with warm-start across blocks (model and optimizer state are carried over from the previous prefix; no cold restarts).

\textbf{Prequential MDL layer.} We run the same prequential coding procedure as on CIFAR-10: the full 50k-sample training set is partitioned into $\fix{K} \le 40$ blocks whose cumulative prefix sizes grow geometrically from $\max(\lceil 0.005\,\sym{m}\rceil, 100) = 250$ to $\sym{m} = 50{,}000$. \fix{The first block $B_1$} is coded with the uniform class prior ($\log_2 100 \approx 6.644$ bits/sample), and for each subsequent block we train on the prefix and evaluate cross-entropy on the held-out block. Per-block epoch budgets follow the same prefix-fraction tiers as on CIFAR-10 ($40/80/120/200$ epochs for $r < 0.05 / < 0.20 / < 0.50 / \ge 0.50$), so the final block again receives a 200-epoch budget. We report $R/\sym{m} = (L_{\mathrm{preq}} - \hat L)/\sym{m}$ as the predictor; targets are $\mathrm{GL} = \mathrm{test\_loss} - \mathrm{train\_loss}$ in bits/sample (computed post hoc as test bits-per-image minus final-train bits-per-image) and $\mathrm{GE} = \mathrm{train\_acc} - \mathrm{test\_acc}$.

\textbf{Setup grid.} The CIFAR-100 grid is built from three orthogonal axes that share a common anchor (ResNet-34, $\mathrm{wd} = 0$, $\mathrm{dropout} = 0$, run for three repeats):
\begin{itemize}
  \item \textbf{Architecture sweep.} $\mathrm{arch} \in \{\text{ResNet-}18, 50, 101, 152\}$ at $\mathrm{wd} = 0$, $\mathrm{dropout} = 0$; $4 \times 3\,(\text{repeat}) = 12$ cells. Together with the anchor's three ResNet-34 repeats, this covers all five depths.
  \item \textbf{Weight-decay sweep.} ResNet-34 at $\mathrm{dropout} = 0$, $\mathrm{wd} \in \{10^{-6},\,3{\cdot}10^{-6},\,10^{-5},\,3{\cdot}10^{-5},\,10^{-4},\,3{\cdot}10^{-4},\,10^{-3},\,3{\cdot}10^{-3},\,10^{-2}\}$; $9 \times 3 = 27$ cells. Combined with the anchor at $\mathrm{wd} = 0$, this gives 10 weight-decay values in total.
  \item \textbf{Dropout sweep.} ResNet-34 at $\mathrm{wd} = 0$, $\mathrm{dropout}$ on a 19-point uniform grid in $(0, 0.5]$; $19 \times 3 = 57$ cells. Combined with the anchor at $\mathrm{dropout} = 0$, this gives a 20-point uniform grid on $[0, 0.5]$.
\end{itemize}
The three sweeps together with the shared anchor yield $12 + 27 + 57 + 3 = 99$ cells.

\subsection{Computation Resources} \label{app:computation_resources}

All experiments were executed on NVIDIA A100 GPUs via SLURM array jobs on a shared HPC cluster. Each individual run was allocated:
\begin{itemize}
  \item \textbf{GPU}: 1 × NVIDIA A100 (80 GB GPU memory)
  \item \textbf{CPU}: 4 CPU cores
  \item \textbf{Host memory}: 16 GB
  \item \textbf{Wall-clock time limits}: 24 hours per task for weight-decay and dropout sweeps; 48 hours per task for the architecture sweep (ResNet-152 training takes longer).
\end{itemize}

The CIFAR-10 grid (540 cells, including all subsampling instances) was executed similarly: on the same hardware, with 24--48 hour allocations depending on the run complexity. For CIFAR-100, the three sweeps (architecture, weight decay, dropout) were orchestrated as separate array jobs: the architecture sweep ran up to 5 tasks in parallel (\texttt{--array=0-4}), and the weight-decay and dropout sweeps each ran up to 20 tasks in parallel (\texttt{--array=0-19}). Most individual runs completed well within their wall-clock allocations; the 48-hour limit for the architecture sweep served as a safety margin for the largest models (ResNet-152) under variable cluster load. Typical single-run wall time was 8--16 hours for CIFAR-100 and 6--12 hours for CIFAR-10 (30\%-subsample), depending on architecture size and weight-decay value.

\section*{NeurIPS Paper Checklist}
\begin{enumerate}

\item {\bf Claims}
    \item[] Question: Do the main claims made in the abstract and introduction accurately reflect the paper's contributions and scope?
    \item[] Answer: \answerYes{}
    \item[] Justification: The abstract and introduction (Section~1) state three contributions that map directly to the paper's body: connecting MDL parameter compression to the algorithm-level information bottleneck (Section~3), proving that the prequential regret upper-bounds the description length of the trained parameter under the algorithm-induced distribution in the nearly-interpolating regime (Section~4), turning this into a PAC-MDL generalization bound (Section~5), and empirically tracking the generalization gap on CIFAR-10/100 with Pearson correlations exceeding 0.9 (Section~6). Limitations on the regime of validity and on the transductive setting are flagged in Section~7.
    \item[] Guidelines:
    \begin{itemize}
        \item The answer \answerNA{} means that the abstract and introduction do not include the claims made in the paper.
        \item The abstract and/or introduction should clearly state the claims made, including the contributions made in the paper and important assumptions and limitations. A \answerNo{} or \answerNA{} answer to this question will not be perceived well by the reviewers. 
        \item The claims made should match theoretical and experimental results, and reflect how much the results can be expected to generalize to other settings. 
        \item It is fine to include aspirational goals as motivation as long as it is clear that these goals are not attained by the paper. 
    \end{itemize}

\item {\bf Limitations}
    \item[] Question: Does the paper discuss the limitations of the work performed by the authors?
    \item[] Answer: \answerYes{}
    \item[] Justification: Section~7 (``Limitations and future work'') notes that our generalization bound lives in the transductive PAC-MDL setting and that lifting it to a population-risk statement is left open, and that the empirical validation is restricted to ResNets on CIFAR-10/100 in the nearly-interpolating regime. The interpolation-genericity assumption (Equation~7) and the discrete (finite-precision) parameter space requirement are stated explicitly in Theorems~4.1, 5.3, and 5.5, and Remark~4.2 discusses when these conditions hold or fail. Section~6 also reports outliers outside the nearly-interpolating regime where the bound is not expected to apply.
    \item[] Guidelines:
    \begin{itemize}
        \item The answer \answerNA{} means that the paper has no limitation while the answer \answerNo{} means that the paper has limitations, but those are not discussed in the paper. 
        \item The authors are encouraged to create a separate ``Limitations'' section in their paper.
        \item The paper should point out any strong assumptions and how robust the results are to violations of these assumptions (e.g., independence assumptions, noiseless settings, model well-specification, asymptotic approximations only holding locally). The authors should reflect on how these assumptions might be violated in practice and what the implications would be.
        \item The authors should reflect on the scope of the claims made, e.g., if the approach was only tested on a few datasets or with a few runs. In general, empirical results often depend on implicit assumptions, which should be articulated.
        \item The authors should reflect on the factors that influence the performance of the approach. For example, a facial recognition algorithm may perform poorly when image resolution is low or images are taken in low lighting. Or a speech-to-text system might not be used reliably to provide closed captions for online lectures because it fails to handle technical jargon.
        \item The authors should discuss the computational efficiency of the proposed algorithms and how they scale with dataset size.
        \item If applicable, the authors should discuss possible limitations of their approach to address problems of privacy and fairness.
        \item While the authors might fear that complete honesty about limitations might be used by reviewers as grounds for rejection, a worse outcome might be that reviewers discover limitations that aren't acknowledged in the paper. The authors should use their best judgment and recognize that individual actions in favor of transparency play an important role in developing norms that preserve the integrity of the community. Reviewers will be specifically instructed to not penalize honesty concerning limitations.
    \end{itemize}

\item {\bf Theory assumptions and proofs}
    \item[] Question: For each theoretical result, does the paper provide the full set of assumptions and a complete (and correct) proof?
    \item[] Answer: \answerYes{}
    \item[] Justification: All theoretical results (Proposition~3.1, Theorem~4.1, Theorem~5.1, Theorem~5.3, Theorem~5.5) are numbered, cross-referenced, and stated with their assumptions in the main text---e.g., the interpolation-genericity condition in Equation~(7), the discrete/finite-precision parameter space, and the transductive PAC-MDL setup. Complete proofs are given in Appendix~A (Sections~A.1--A.4), with the corresponding intuition and proof sketches provided for the main theorems.
    \item[] Guidelines:
    \begin{itemize}
        \item The answer \answerNA{} means that the paper does not include theoretical results. 
        \item All the theorems, formulas, and proofs in the paper should be numbered and cross-referenced.
        \item All assumptions should be clearly stated or referenced in the statement of any theorems.
        \item The proofs can either appear in the main paper or the supplemental material, but if they appear in the supplemental material, the authors are encouraged to provide a short proof sketch to provide intuition. 
        \item Inversely, any informal proof provided in the core of the paper should be complemented by formal proofs provided in appendix or supplemental material.
        \item Theorems and Lemmas that the proof relies upon should be properly referenced. 
    \end{itemize}

    \item {\bf Experimental result reproducibility}
    \item[] Question: Does the paper fully disclose all the information needed to reproduce the main experimental results of the paper to the extent that it affects the main claims and/or conclusions of the paper (regardless of whether the code and data are provided or not)?
    \item[] Answer: \answerYes{}
    \item[] Justification: Section~6 describes the empirical prequential coding scheme, the partition into $T=40$ blocks of geometrically growing size, the architectures (ResNet family), and the datasets (CIFAR-10/100). Appendix~B gives the full reproduction recipe: subsampling protocol, normalization and augmentation, optimizer (SGD with momentum, learning-rate schedule, weight decay grid, batch sizes, label smoothing), the prequential block schedule and per-block epoch budgets, the warm-start protocol, the full $3\times 3\times 3\times 4\times 5 = 540$-cell CIFAR-10 grid, and the $15+30+60-3 = 99$-cell CIFAR-100 grid (architecture, weight-decay, dropout sweeps). An anonymized code release is also provided (see next item).
    \item[] Guidelines:
    \begin{itemize}
        \item The answer \answerNA{} means that the paper does not include experiments.
        \item If the paper includes experiments, a \answerNo{} answer to this question will not be perceived well by the reviewers: Making the paper reproducible is important, regardless of whether the code and data are provided or not.
        \item If the contribution is a dataset and\slash or model, the authors should describe the steps taken to make their results reproducible or verifiable. 
        \item Depending on the contribution, reproducibility can be accomplished in various ways. For example, if the contribution is a novel architecture, describing the architecture fully might suffice, or if the contribution is a specific model and empirical evaluation, it may be necessary to either make it possible for others to replicate the model with the same dataset, or provide access to the model. In general. releasing code and data is often one good way to accomplish this, but reproducibility can also be provided via detailed instructions for how to replicate the results, access to a hosted model (e.g., in the case of a large language model), releasing of a model checkpoint, or other means that are appropriate to the research performed.
        \item While NeurIPS does not require releasing code, the conference does require all submissions to provide some reasonable avenue for reproducibility, which may depend on the nature of the contribution. For example
        \begin{enumerate}
            \item If the contribution is primarily a new algorithm, the paper should make it clear how to reproduce that algorithm.
            \item If the contribution is primarily a new model architecture, the paper should describe the architecture clearly and fully.
            \item If the contribution is a new model (e.g., a large language model), then there should either be a way to access this model for reproducing the results or a way to reproduce the model (e.g., with an open-source dataset or instructions for how to construct the dataset).
            \item We recognize that reproducibility may be tricky in some cases, in which case authors are welcome to describe the particular way they provide for reproducibility. In the case of closed-source models, it may be that access to the model is limited in some way (e.g., to registered users), but it should be possible for other researchers to have some path to reproducing or verifying the results.
        \end{enumerate}
    \end{itemize}

\item {\bf Open access to data and code}
    \item[] Question: Does the paper provide open access to the data and code, with sufficient instructions to faithfully reproduce the main experimental results, as described in supplemental material?
    \item[] Answer: \answerYes{}
    \item[] Justification: An anonymized code repository is linked in Section~6 (\url{https://anonymous.4open.science/r/mdl_gen-E317}), containing the prequential coding implementation and training scripts. The datasets used (CIFAR-10 and CIFAR-100) are publicly available standard benchmarks. Hyperparameters, schedules, and prequential blocking details are documented in Appendix~B in sufficient detail to reproduce the experimental grids.
    \item[] Guidelines:
    \begin{itemize}
        \item The answer \answerNA{} means that paper does not include experiments requiring code.
        \item Please see the NeurIPS code and data submission guidelines (\url{https://neurips.cc/public/guides/CodeSubmissionPolicy}) for more details.
        \item While we encourage the release of code and data, we understand that this might not be possible, so \answerNo{} is an acceptable answer. Papers cannot be rejected simply for not including code, unless this is central to the contribution (e.g., for a new open-source benchmark).
        \item The instructions should contain the exact command and environment needed to run to reproduce the results. See the NeurIPS code and data submission guidelines (\url{https://neurips.cc/public/guides/CodeSubmissionPolicy}) for more details.
        \item The authors should provide instructions on data access and preparation, including how to access the raw data, preprocessed data, intermediate data, and generated data, etc.
        \item The authors should provide scripts to reproduce all experimental results for the new proposed method and baselines. If only a subset of experiments are reproducible, they should state which ones are omitted from the script and why.
        \item At submission time, to preserve anonymity, the authors should release anonymized versions (if applicable).
        \item Providing as much information as possible in supplemental material (appended to the paper) is recommended, but including URLs to data and code is permitted.
    \end{itemize}

\item {\bf Experimental setting/details}
    \item[] Question: Does the paper specify all the training and test details (e.g., data splits, hyperparameters, how they were chosen, type of optimizer) necessary to understand the results?
    \item[] Answer: \answerYes{}
    \item[] Justification: Section~6 specifies the high-level setup (datasets, architecture families, prequential blocking, sweep over depth, batch size, weight decay, dropout, seed). Appendix~B.1 (CIFAR-10) and B.2 (CIFAR-100) give the data splits, normalization, augmentation, optimizer settings (SGD momentum, learning-rate schedules), weight-decay grids, batch sizes, dropout grids, prequential block schedule, per-block epoch budgets, and the rationale for each grid choice (matching the IB baseline of Kawaguchi et al.~[25] on CIFAR-10).
    \item[] Guidelines:
    \begin{itemize}
        \item The answer \answerNA{} means that the paper does not include experiments.
        \item The experimental setting should be presented in the core of the paper to a level of detail that is necessary to appreciate the results and make sense of them.
        \item The full details can be provided either with the code, in appendix, or as supplemental material.
    \end{itemize}

\item {\bf Experiment statistical significance}
    \item[] Question: Does the paper report error bars suitably and correctly defined or other appropriate information about the statistical significance of the experiments?
    \item[] Answer: \answerNo{}
    \item[] Justification: Our central empirical claim is that the prequential regret tracks the generalization gap across many configurations, and we substantiate it with three rank/linear correlation statistics (Pearson, Spearman, Kendall) computed across the full 540-cell (CIFAR-10) and 99-cell (CIFAR-100) hyperparameter grids, together with scatter plots that display all per-configuration points (Table~1 and Figure~1). We do not report confidence intervals on the correlation values themselves; the main source of variability---across architecture, weight decay, dropout, batch size, and seed---is reflected in the spread of points entering the correlation, not in error bars on a single number.
    \item[] Guidelines:
    \begin{itemize}
        \item The answer \answerNA{} means that the paper does not include experiments.
        \item The authors should answer \answerYes{} if the results are accompanied by error bars, confidence intervals, or statistical significance tests, at least for the experiments that support the main claims of the paper.
        \item The factors of variability that the error bars are capturing should be clearly stated (for example, train/test split, initialization, random drawing of some parameter, or overall run with given experimental conditions).
        \item The method for calculating the error bars should be explained (closed form formula, call to a library function, bootstrap, etc.)
        \item The assumptions made should be given (e.g., Normally distributed errors).
        \item It should be clear whether the error bar is the standard deviation or the standard error of the mean.
        \item It is OK to report 1-sigma error bars, but one should state it. The authors should preferably report a 2-sigma error bar than state that they have a 96\% CI, if the hypothesis of Normality of errors is not verified.
        \item For asymmetric distributions, the authors should be careful not to show in tables or figures symmetric error bars that would yield results that are out of range (e.g., negative error rates).
        \item If error bars are reported in tables or plots, the authors should explain in the text how they were calculated and reference the corresponding figures or tables in the text.
    \end{itemize}

\item {\bf Experiments compute resources}
    \item[] Question: For each experiment, does the paper provide sufficient information on the computer resources (type of compute workers, memory, time of execution) needed to reproduce the experiments?
    \item[] Answer: \answerYes{}
    \item[] Justification: Section~6 reports that 540 networks were trained on CIFAR-10 and 99 on CIFAR-100, and that the CIFAR-100 evaluation was constrained by a 48-hour compute budget---which itself motivated the smaller grid. Appendix~B specifies the full per-cell training cost (200-epoch budget per cell, prequential block schedule, batch sizes, optimizer settings), and also reports the compute resources (GPU, CPU, host memory, and wall-clock time limits). 
    \item[] Guidelines:
    \begin{itemize}
        \item The answer \answerNA{} means that the paper does not include experiments.
        \item The paper should indicate the type of compute workers CPU or GPU, internal cluster, or cloud provider, including relevant memory and storage.
        \item The paper should provide the amount of compute required for each of the individual experimental runs as well as estimate the total compute. 
        \item The paper should disclose whether the full research project required more compute than the experiments reported in the paper (e.g., preliminary or failed experiments that didn't make it into the paper). 
    \end{itemize}
    
\item {\bf Code of ethics}
    \item[] Question: Does the research conducted in the paper conform, in every respect, with the NeurIPS Code of Ethics \url{https://neurips.cc/public/EthicsGuidelines}?
    \item[] Answer: \answerYes{}
    \item[] Justification: The work is theoretical and uses only standard public image-classification benchmarks (CIFAR-10/100); it involves no human subjects, no personally identifiable data, and no model release with foreseeable misuse risk. We have reviewed the NeurIPS Code of Ethics and the research conforms to it; anonymity is preserved in the submission (anonymized code URL).
    \item[] Guidelines:
    \begin{itemize}
        \item The answer \answerNA{} means that the authors have not reviewed the NeurIPS Code of Ethics.
        \item If the authors answer \answerNo, they should explain the special circumstances that require a deviation from the Code of Ethics.
        \item The authors should make sure to preserve anonymity (e.g., if there is a special consideration due to laws or regulations in their jurisdiction).
    \end{itemize}

\item {\bf Broader impacts}
    \item[] Question: Does the paper discuss both potential positive societal impacts and negative societal impacts of the work performed?
    \item[] Answer: \answerNA{}
    \item[] Justification: This is foundational theoretical work on understanding generalization in over-parameterized neural networks via the MDL principle. It does not propose a deployable system, a generative model, or any application tied to a specific user-facing technology, and we are not aware of a direct path from the prequential regret as a generalization diagnostic to a specific positive or negative societal impact beyond the generic effects of progress in learning theory.
    \item[] Guidelines:
    \begin{itemize}
        \item The answer \answerNA{} means that there is no societal impact of the work performed.
        \item If the authors answer \answerNA{} or \answerNo, they should explain why their work has no societal impact or why the paper does not address societal impact.
        \item Examples of negative societal impacts include potential malicious or unintended uses (e.g., disinformation, generating fake profiles, surveillance), fairness considerations (e.g., deployment of technologies that could make decisions that unfairly impact specific groups), privacy considerations, and security considerations.
        \item The conference expects that many papers will be foundational research and not tied to particular applications, let alone deployments. However, if there is a direct path to any negative applications, the authors should point it out. For example, it is legitimate to point out that an improvement in the quality of generative models could be used to generate Deepfakes for disinformation. On the other hand, it is not needed to point out that a generic algorithm for optimizing neural networks could enable people to train models that generate Deepfakes faster.
        \item The authors should consider possible harms that could arise when the technology is being used as intended and functioning correctly, harms that could arise when the technology is being used as intended but gives incorrect results, and harms following from (intentional or unintentional) misuse of the technology.
        \item If there are negative societal impacts, the authors could also discuss possible mitigation strategies (e.g., gated release of models, providing defenses in addition to attacks, mechanisms for monitoring misuse, mechanisms to monitor how a system learns from feedback over time, improving the efficiency and accessibility of ML).
    \end{itemize}
    
\item {\bf Safeguards}
    \item[] Question: Does the paper describe safeguards that have been put in place for responsible release of data or models that have a high risk for misuse (e.g., pre-trained language models, image generators, or scraped datasets)?
    \item[] Answer: \answerNA{}
    \item[] Justification: The paper releases only research code for computing the prequential regret on standard public image-classification benchmarks; no pretrained generative models, language models, or scraped datasets with misuse risk are released.
    \item[] Guidelines:
    \begin{itemize}
        \item The answer \answerNA{} means that the paper poses no such risks.
        \item Released models that have a high risk for misuse or dual-use should be released with necessary safeguards to allow for controlled use of the model, for example by requiring that users adhere to usage guidelines or restrictions to access the model or implementing safety filters. 
        \item Datasets that have been scraped from the Internet could pose safety risks. The authors should describe how they avoided releasing unsafe images.
        \item We recognize that providing effective safeguards is challenging, and many papers do not require this, but we encourage authors to take this into account and make a best faith effort.
    \end{itemize}

\item {\bf Licenses for existing assets}
    \item[] Question: Are the creators or original owners of assets (e.g., code, data, models), used in the paper, properly credited and are the license and terms of use explicitly mentioned and properly respected?
    \item[] Answer: \answerYes{}
    \item[] Justification: All existing assets used are cited: CIFAR-10/100 (Krizhevsky~[26]), the ResNet/PreResNet architectures (He et al.~[21,~22]), the IB baseline pipeline of Kawaguchi et al.~[25] whose released codebase we reuse for the CIFAR-10 subsampling protocol (Appendix~B.1), the SWAG posterior approximation (Maddox et al.~[29]), the prequential coding scheme of Blier and Ollivier~[7], and the PAC-MDL and transductive PAC-Bayes frameworks (Blum and Langford~[8]; B\'egin et al.~[5]). These are standard research assets used under their public terms and within their intended research use.
    \item[] Guidelines:
    \begin{itemize}
        \item The answer \answerNA{} means that the paper does not use existing assets.
        \item The authors should cite the original paper that produced the code package or dataset.
        \item The authors should state which version of the asset is used and, if possible, include a URL.
        \item The name of the license (e.g., CC-BY 4.0) should be included for each asset.
        \item For scraped data from a particular source (e.g., website), the copyright and terms of service of that source should be provided.
        \item If assets are released, the license, copyright information, and terms of use in the package should be provided. For popular datasets, \url{paperswithcode.com/datasets} has curated licenses for some datasets. Their licensing guide can help determine the license of a dataset.
        \item For existing datasets that are re-packaged, both the original license and the license of the derived asset (if it has changed) should be provided.
        \item If this information is not available online, the authors are encouraged to reach out to the asset's creators.
    \end{itemize}

\item {\bf New assets}
    \item[] Question: Are new assets introduced in the paper well documented and is the documentation provided alongside the assets?
    \item[] Answer: \answerYes{}
    \item[] Justification: We release research code implementing the prequential coding procedure and the experimental pipeline at the anonymized URL given in Section~6 (\url{https://anonymous.4open.science/r/mdl_gen-E317}). Documentation of the algorithm and the configurations it reproduces is provided in Section~6 and Appendix~B; no new datasets or pretrained models are released.
    \item[] Guidelines:
    \begin{itemize}
        \item The answer \answerNA{} means that the paper does not release new assets.
        \item Researchers should communicate the details of the dataset\slash code\slash model as part of their submissions via structured templates. This includes details about training, license, limitations, etc. 
        \item The paper should discuss whether and how consent was obtained from people whose asset is used.
        \item At submission time, remember to anonymize your assets (if applicable). You can either create an anonymized URL or include an anonymized zip file.
    \end{itemize}

\item {\bf Crowdsourcing and research with human subjects}
    \item[] Question: For crowdsourcing experiments and research with human subjects, does the paper include the full text of instructions given to participants and screenshots, if applicable, as well as details about compensation (if any)?
    \item[] Answer: \answerNA{}
    \item[] Justification: The paper does not involve crowdsourcing or research with human subjects; all experiments are on standard public image-classification benchmarks.
    \item[] Guidelines:
    \begin{itemize}
        \item The answer \answerNA{} means that the paper does not involve crowdsourcing nor research with human subjects.
        \item Including this information in the supplemental material is fine, but if the main contribution of the paper involves human subjects, then as much detail as possible should be included in the main paper. 
        \item According to the NeurIPS Code of Ethics, workers involved in data collection, curation, or other labor should be paid at least the minimum wage in the country of the data collector. 
    \end{itemize}

\item {\bf Institutional review board (IRB) approvals or equivalent for research with human subjects}
    \item[] Question: Does the paper describe potential risks incurred by study participants, whether such risks were disclosed to the subjects, and whether Institutional Review Board (IRB) approvals (or an equivalent approval/review based on the requirements of your country or institution) were obtained?
    \item[] Answer: \answerNA{}
    \item[] Justification: The paper does not involve research with human subjects, so IRB approval is not applicable.
    \item[] Guidelines:
    \begin{itemize}
        \item The answer \answerNA{} means that the paper does not involve crowdsourcing nor research with human subjects.
        \item Depending on the country in which research is conducted, IRB approval (or equivalent) may be required for any human subjects research. If you obtained IRB approval, you should clearly state this in the paper. 
        \item We recognize that the procedures for this may vary significantly between institutions and locations, and we expect authors to adhere to the NeurIPS Code of Ethics and the guidelines for their institution. 
        \item For initial submissions, do not include any information that would break anonymity (if applicable), such as the institution conducting the review.
    \end{itemize}

\item {\bf Declaration of LLM usage}
    \item[] Question: Does the paper describe the usage of LLMs if it is an important, original, or non-standard component of the core methods in this research? Note that if the LLM is used only for writing, editing, or formatting purposes and does \emph{not} impact the core methodology, scientific rigor, or originality of the research, declaration is not required.
    %this research?
    \item[] Answer: \answerNA{}
    \item[] Justification: LLMs are not part of the core methodology of this paper; the contributions are an MDL-based generalization analysis and image-classification experiments on CIFAR-10/100. Any use of LLMs was limited to writing/editing assistance and does not affect scientific content.
    \item[] Guidelines:
    \begin{itemize}
        \item The answer \answerNA{} means that the core method development in this research does not involve LLMs as any important, original, or non-standard components.
        \item Please refer to our LLM policy in the NeurIPS handbook for what should or should not be described.
    \end{itemize}

\end{enumerate}

\end{document}